\documentclass[lettersize,journal]{IEEEtran}
\usepackage{amsmath,amsfonts}
\usepackage{algorithmic}
\usepackage{algorithm}
\usepackage{array}
\usepackage[caption=false,font=normalsize,labelfont=sf,textfont=sf]{subfig}
\usepackage{textcomp}
\usepackage{stfloats}
\usepackage{url}
\usepackage{verbatim}
\usepackage{graphicx}
\usepackage{cite}
\usepackage{url}
\usepackage{xcolor}
\usepackage{amssymb}
\usepackage{amsthm}
\usepackage{multirow}
\usepackage{multicol}
\usepackage{fdsymbol} 
\usepackage{centernot}

\newtheorem{theorem}{Theorem}
\newtheorem{definition}{Definition}
\newtheorem{lemma}{Lemma}
\newtheorem{property}{Property}

\newtheorem{proposition}{Proposition}

\newcommand{\eat}[1]{}
\newcommand{\etal}{{\em et al.~}}       
\newcommand{\ie}{{\em i.e.,~}}           
\newcommand{\aka}{{a.k.a.~}}       
\newcommand{\wrt}{{\em wrt.,~}}         

\def\circlerightarrow{\put(2,2.5){\circle{2.5}}\rightarrow}

\def\astrightarrow{\put(0.2,-2.2){*}\rightarrow}
\def\astleftarrow{\leftarrow\put(-5,-2.2){*}}

\begin{document}

\title{Discovering Ancestral Instrumental Variables for Causal Inference from Observational Data}

\author{Debo~Cheng~and~Jiuyong~Li~and~Lin~Liu~and~Kui~Yu~and~Thuc~Duy~Le~and~Jixue~Liu
	\IEEEcompsocitemizethanks{\IEEEcompsocthanksitem D. Cheng, J. Li, L. Liu, T.D. Le and J. Liu are with STEM, University of South Australia, Mawson Lakes, South Australia, 5095, Australia (e-mail: \{Debo.Cheng,Jiuyong.Li,Lin.Liu,Thuc.Le,Jixue.Liu\}@unisa.edu.au). 
	K. Yu is with the School of Computer Science and Information Engineering, Hefei University of Technology, Hefei, 230000, China (e-mail: ykui713@gmail.com).
	(Corresponding authors: Debo Cheng and Jiuyong Li)}}

\markboth{Journal of \LaTeX\ Class Files,~Vol.~*, No.~*, August~202*}%
{Shell \MakeLowercase{\textit{et al.}}: A Sample Article Using IEEEtran.cls for IEEE Journals}



\maketitle

\begin{abstract}
	Instrumental variable (IV) is a powerful approach to inferring the causal effect of a treatment on an outcome of interest from observational data even when there exist latent confounders between the treatment and the outcome. However, existing IV methods require that an IV is selected and justified with domain knowledge. An invalid IV may lead to biased estimates. Hence, discovering a valid IV is critical to the applications of IV methods. In this paper, we study and design a data-driven algorithm to discover valid IVs from data under mild assumptions. We develop the theory based on partial ancestral graphs (PAGs) to support the search for a set of candidate Ancestral IVs (AIVs), and for each possible AIV, the identification of its conditioning set. Based on the theory, we propose a data-driven  algorithm to discover a pair of   IVs from data. The experiments on synthetic and real-world datasets show that the developed IV discovery algorithm estimates accurate estimates of causal effects in comparison with the state-of-the-art IV based causal effect estimators.
\end{abstract}


\section{Introduction}
\label{Sec:Intro}
\IEEEPARstart{A} fundamental challenge for inferring from observational data the causal effect of a treatment $W$ (a.k.a. exposure, intervention or action) on an outcome $Y$ of interest is the presence of latent (\aka unobserved or unmeasured) confounders, variables which affect $W$ and $Y$ simultaneously. Instrumental variables (IVs) are a powerful tool to address this challenge, primarily used by statisticians, economists and social scientist~\cite{reiersol1945confluence,imbens2015causal,blalock2017causal}. It is possible to eliminate the confounding bias by leveraging a valid IV~\cite{angrist1996identification}.

The standard IV approach requires a predefined IV (denoted as $S$) that meets the following three conditions: (1) $S$ is a cause of $W$, (2) There is no confounding bias for the effect of $S$ on $Y$ (\aka exogeneity), and (3) the effect of $S$ on $Y$ is entirely mediated through $W$ (a.k.a. the exclusion restriction)~\cite{martens2006instrumental}. For example, $S$ in the causal graph in Fig.~\ref{fig:twocausaldiagrams} (a) is a standard IV as it meets all the three conditions. The second and third conditions have to be justified by domain knowledge~\cite{pearl2009causality} and hence IV based methods are mostly classified as empirical methods in literature~\cite{kuroki2005instrumental,silva2017learning}. 

\begin{figure}[t]
	\begin{center}
		\includegraphics[scale=0.55]{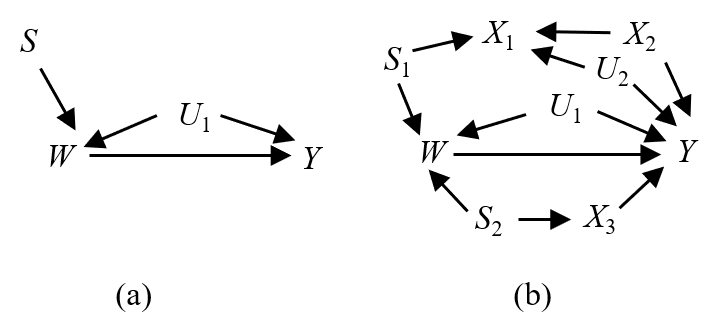}
		\caption{(a) $S$ is a standard IV for $W\rightarrow Y$. (b) $S_1$ and $S_2$ are valid IVs \wrt $W\rightarrow Y$, conditioning on $\mathbf{Z}=\{X_2, X_3\}$.}
		\label{fig:twocausaldiagrams}
	\end{center}
\end{figure}

Data-driven approach to identifying standard IVs is impractical. Instrumental inequality~\cite{pearl1995testability,kuroki2005instrumental} has been proposed to test an IV in data, but it is a necessary condition, not a sufficient condition. Under a set of strong assumptions on the data distribution, a semi-instrumental variable that can be tested in data has been introduced by~\cite{chu2001semi}. However, there is not a practical algorithm rooted from the above concepts. By assuming that at least a half of the covariates are valid IVs, Kang \etal~\cite{kang2016instrumental} proposed an algorithm, sisVIVE (some invalid, some valid IV estimator), to estimate the causal effect in data. Hartford \etal~\cite{hartford2021valid} extended the sisVIVE algorithm by employing a deep learning based IV estimator. The main challenge of the two algorithms is that their strong assumption is often unsatisfied in many real-world applications. 

To relax the last two conditions (\ie the exogeneity and exclusion restriction) of a standard IV, a graphical criterion~\cite{pearl2009causality,brito2002generalized} is proposed to identify an observed variable $S$ as an IV (\ie the conditional IV (CIV)), conditioning on a set of observed variables $\mathbf{Z}$ from a given DAG (directed acyclic graph, which represents the causal relations of all measured and unmeasured variables). CIV allows a confounding bias between $S$ and $Y$, and $S$ have multiple causal paths to $Y$. The bias can be adjusted by a conditioning set. Van der Zander \etal~\cite{van2015efficiently} have revised CIV to Ancestral IV (AIV) to avoid the situation where void CIVs may be identified based on the original definition (see Section~\ref{subsec:IVmethods} for details). AIV identification needs a DAG too.

IV.tetrad~\cite{silva2017learning} is the only existing data-driven CIV method. IV.tetrad based on CIV and it requires two valid CIVs in the covariate set. The tetrad condition is used to discover the pair of   CIVs in data and it assumes that the conditioning sets of the pair of CIVs are the same and equal to the set of remaining covariates (i.e. the original covariate set excluding the pair of CIVs). This tetrad condition leads to a wrong identification when the set of  remaining covariates contains a collider. For example, $S_1$ and $S_2$ in Fig.~\ref{fig:twocausaldiagrams} (b) are a pair of   CIVs. The conditioning sets for $S_1$ and $S_2$ are the same, both equal to $\{X_2, X_3\}$, but not as assumed in the IV.tetrad, i.e. $\{X_1, X_2, X_3\}$. When $X_1$ is used in the conditioning set, path $ S_1 \to X_1 \leftarrow U_2 \to Y$ is opened and $S_1$ does not instrumentalise $W$ conditioning on $\{X_1, X_2, X_3\}$ any more. Hence, the tetrad condition does not find the right CIV pairs and leads to a biased causal effect estimation.    

This work improves IV.tetrad in the following ways, which are also our contributions. 
\begin{enumerate}
	\item We generalise the tetrad condition so that each AIV in the pair conditions on its own conditioning set, and this rectifies the current tetrad condition which fails to find the right pair when the covariate set contains a collider.   
	\item We develop the theory for identifying the set of candidate AIVs in a reduced space for efficient search for a pair of AIVs.
	\item We propose a data-driven algorithm for estimating causal effects from data with latent variables based on the above developed theorems. Extensive experiments on synthetic and real-world datasets have shown the effectiveness of the algorithm.  
\end{enumerate}

\section{Background}
\label{sec:pre}
\subsection{Graph Terminology}
\label{subsec:graphicalnotations}
A graph $\mathcal{G}\!\!=\!\!(\mathbf{V}, \mathbf{E})$ is composed of a set of nodes $\mathbf{V}\!\!=\!\!\{V_{1}, \dots, V_{p}\}$, representing random variables, and a set of edges $\mathbf{E} \!\subseteq \mathbf{V} \times \mathbf{V}$, representing the relations between nodes. In this paper, we assume that in a graph  $\mathcal{G}$, there is at most one edge between any two nodes.

Two nodes are adjacent if there exists an edge between them. For an edge $V_i\rightarrow V_j$, $V_i$ and $V_j$ are its head and tail respectively and $V_i$ is known as a parent of $V_j$ (and $V_j$ is a child of $V_i$). We use $Adj(V)$, $Pa(V)$ and $Ch(V)$ to denote the sets of all adjacent nodes, parents and children of $V$, respectively. A path $\pi$ is a sequence of nodes $\left \langle V_1, \dots, V_n\right \rangle$ such that for $1\leq i \leq n-1$, the pair $(V_i, V_{i+1})$ is adjacent. A path $\pi$ from $V_i$ to $V_j$ is a directed or causal path if all edges along it are directed towards $V_j$, and $(V_i, V_j)$ are called endpoint nodes, other nodes are non-endpoint nodes. If there is a directed path $\pi$ from $V_i$ to $V_j$, $V_i$ is a known as an ancestor of $V_j$ and $V_j$ is a descendant of $V_i$. The sets of ancestors and descendants of a node $V$ are denoted as $An(V)$ and $De(V)$, respectively. 

A DAG (directed acyclic graph) is a direct graph (i.e. a graph containing only directed edges $\rightarrow$) without directed cycles (i.e. a directed path whose two endpoints are the same node). A DAG is often used to represent the data generation mechanism or causal mechanism underlying the data, with all variables, both observed and unobserved (if any) included in the graph.

Ancestral graphs are used to represent the data generating generation mechanisms that may involve latent variables, with only observed variables included in the graphs~\cite{richardson2002ancestral,zhang2008causal}. An ancestral graph may contain three types of edges, $\rightarrow$, $\leftrightarrow$ (it is used to represent that a common cause of two observed variables is a latent variable.) and $\circlerightarrow$ (the circle tail denotes the orientation of the edge is uncertain.), and we use `$\astrightarrow$' to denote any of the three types. $C$ is a collider on the path $\pi$ if $\pi$ contains a subpath  $\astrightarrow C\astleftarrow$. In an ancestral graph, a path is a collider path if every non-endpoint node on it is a collider. A path of length one is a trivial collider path.

In a graph, an almost directed cycle occurs when $V_i \leftrightarrow V_j$ is in the graph and $V_j\in An(V_i)$. An ancestral graph is a graph that does not contain directed cycles or almost directed cycles~\cite{richardson2002ancestral}. In an ancestral graph, a path from $V_i$ to $V_j$ is a possibly directed or causal path if there is not an arrowhead pointing in the direction of $V_i$. In this case, $V_i$ is a possible ancestor of $V_j$ and $V_j$ is a possible descendant of $V_i$. The sets of possible ancestors and descendants of $V$ are denoted as $PossAn(V)$ and $PossDe(V)$, respectively.

In graphical causal modelling, the assumptions of Markov property and faithfulness are often involved to discuss the relationship between the causal graph and the data distribution. 

\begin{definition}[Markov property~\cite{pearl2009causality}]
	\label{Markov condition}
	Given a DAG $\mathcal{G}=(\mathbf{V}, \mathbf{E})$ and the joint probability distribution of $\mathbf{V}$ $(prob(\mathbf{V}))$, $\mathcal{G}$ satisfies the Markov property if for $\forall V_i \in \mathbf{V}$, $V_i$ is probabilistically independent of all of its non-descendants, given $Pa(V_i)$.
\end{definition}

\begin{definition}[Faithfulness~\cite{spirtes2000causation}]
	\label{Faithfulness}
	A DAG $\mathcal{G}=(\mathbf{V}, \mathbf{E})$ is faithful to a joint distribution $prob(\mathbf{V})$ over the set of variables $\mathbf{V}$ if and only if every independence present in $prob(\mathbf{V})$ is entailed by $\mathcal{G}$ and satisfies the Markov property. A joint distribution $prob(\mathbf{V})$ over the set of variables $\mathbf{V}$ is faithful to the DAG $\mathcal{G}$ if and only if the DAG $\mathcal{G}$ is faithful to the joint distribution $prob(\mathbf{V})$.
\end{definition}
A DAG $\mathcal{G}=(\mathbf{V}, \mathbf{E})$ satisfies both assumptions, the probability distribution $prob(\mathbf{V})$ can be factorised as: $prob(\mathbf{V}) = \prod_{i}^{p}prob(V_i |Pa(V_i))$. Thus, together Markov property and faithfulness establish a close relation between the causal graph and the data distribution.

\begin{definition}[Causal sufficiency~\cite{spirtes2000causation}]
	\label{Causalsufficiency}
	In a data, for every pair of observed variables $(V_i, V_j)$ in $\mathbf{V}$, all their common causes are also in $\mathbf{V}$.
\end{definition}

In a DAG, d-separation is a well-known graphical criterion that is used to read off the identification of conditional independence between variables entailed in the DAG when the Markov property, faithfulness and causal sufficiency are satisfied~\cite{pearl2009causality,spirtes2000causation}.

\begin{definition}[d-separation~\cite{pearl2009causality}]
	\label{d-separation}
	A path $\pi$ in a DAG $\mathcal{G}=(\mathbf{V}, \mathbf{E})$ is said to be d-separated (or blocked) by a set of nodes $\mathbf{Z}$ if and only if
	(i) $\pi$ contains a chain $V_i \rightarrow V_k \rightarrow V_j$ or a fork $V_i \leftarrow V_k \rightarrow V_j$ such that the middle node $V_k$ is in $\mathbf{Z}$, or
	(ii) $\pi$ contains a collider $V_k$ such that $V_k$ is not in $\mathbf{Z}$ and no descendant of $V_k$ is in $\mathbf{Z}$.
	A set $\mathbf{Z}$ is said to d-separate $V_i$ from $V_j$ ($ V_i \Vbar_{d} V_j\mid\mathbf{Z}$) if and only if $\mathbf{Z}$ blocks every path between $V_i$ to $V_j$. Otherwise they are said to be d-connected by $\mathbf{Z}$, denoted as $V_i\nVbar_{d} V_j\mid\mathbf{Z}$.
\end{definition}

\begin{property}
	Two observed variables $V_i$ and $V_j$ are d-separated given a conditioning set $\mathbf{Z}$ in a DAG if and only if $V_i$ and $V_j$ are conditionally independence given $\mathbf{Z}$ in data~\cite{spirtes2000causation}.  if $V_i$ and $V_j$ are d-connected, $V_i$ and $V_j$  are conditionally dependent. 
\end{property}

However, a system may involve the latent variables (the latent variable is an unmeasured common cause of two nodes) in most situations since there is not a close world. Ancestral graphs are proposed to represent the system that may involve latent variables~\cite{richardson2002ancestral,zhang2008causal}. In our work, we utilise the Maximal ancestral graph and introduce it as follows.

\begin{definition}[Maximal ancestral graph (MAG)~\cite{richardson2002ancestral}]
	\label{MAG}
	An ancestral graph $\mathcal{M}=(\mathbf{V}, \mathbf{E})$ is a MAG when every pair of non-adjacent nodes $V_{i}$ and $V_{j}$ in $\mathcal{M}$ is m-separated by a set $\mathbf{Z}\subseteq \mathbf{V}\backslash \{V_{i}, V_{j}\}$.
\end{definition}

It is worth noting that a DAG satisfies both conditions of a MAG, so a DAG is also a MAG without bi-directed edges~\cite{zhang2008causal}.
An important concept in a DAG is d-separation, which captures the conditional independence relationships between variables based on Markov property~\cite{pearl2009causality}. A natural extension of the d-separation to an ancestral graph is m-separation~\cite{richardson2002ancestral}.

\begin{definition}[m-separation~\cite{richardson2002ancestral}]
	\label{m-separation}
	In an ancestral graph $\mathcal{M}=(\mathbf{V}, \mathbf{E})$, a path $\pi$ between $V_{i}$ and $V_{j}$ is said to be m-separated by a set of nodes $\mathbf{Z}\subseteq \mathbf{V}\setminus\{V_i, V_j\}$ (possibly $\emptyset$) if $\pi$ contains a subpath $\langle V_l, V_k, V_s\rangle$ such that the middle node $V_k$ is a non-collider on $\pi$ and $V_k \in \mathbf{Z}$; or $\pi$ contains $V_l\astrightarrow V_k\astleftarrow V_s$ such that $V_k \notin \mathbf{Z}$ and no descendant of $V_k$ is in $\mathbf{Z}$. Two nodes $V_{i}$ and $V_{j}$ are said to be m-separated by $\mathbf{Z}$ in $\mathcal{M}$, denoted as $V_i \Vbar_m V_j|\mathbf{Z}$ if every path between $V_{i}$ and $V_{j}$ are m-separated by $\mathbf{Z}$; otherwise they are said to be m-connected by $\mathbf{Z}$, denoted as $V_i\nVbar_m V_j|\mathbf{Z}$.
\end{definition}
\noindent where $\Vbar_m$ denotes m-separation and $\nVbar_m$ denotes m-connecting. In a DAG, m-separation reduces to d-separation. The Markov property of ancestral graph is captured by m-separation. 

If two MAGs represent the same set of m-separations, they are called \emph{Markov equivalent}, and formally, we have the following definition.

\begin{definition}[Markov equivalent MAGs~\cite{zhang2008completeness}]
	\label{markovequ}
	Two MAGs $\mathcal{M}_1$ and $\mathcal{M}_2$ with the same nodes are said to be \emph{Markov equivalent}, denoted $\mathcal{M}_1 \sim \mathcal{M}_2$, if for all triple nodes $X$, $Y$, $Z$, $X$ and $Y$ are m-separated by $Z$ in $\mathcal{M}_1$ if and only if $X$ and $Y$ are m-separated by $Z$ in $\mathcal{M}_2$.
\end{definition}

A set of Markov equivalent MAGs can be encoded uniquely by a \emph{partial ancestral graph} (PAG)~\cite{zhang2008causal}.

\begin{definition}[PAG~\cite{zhang2008causal}]
	\label{def:pag}
	Let $[\mathcal{M}]$ be the Markov equivalence class of a MAG $\mathcal{M}$. The PAG $\mathcal{P}$ for $[\mathcal{M}]$ is a partial mixed graph such that (i). $\mathcal{P}$ has the same adjacent relations among nodes as $\mathcal{M}$ does; (ii). For an edge, its mark of arrowhead or mark of the tail is in $\mathcal{P}$ if and only if the same mark of arrowhead or the same mark of the tail is shared by all MAGs in $[\mathcal{M}]$.
\end{definition}

\begin{definition}[Visibility~\cite{zhang2008causal}]
	\label{Visibility}
	Given a MAG $\mathcal{M}=(\mathbf{V}, \mathbf{E})$, a directed edge $V_{i}\rightarrow V_{j}$ is visible if there is a node $V_{k}\notin Adj(V_{j})$, such that either there is an edge between $V_{k}$ and $V_{i}$ that is into $V_{i}$, or there is a collider path between $V_{k}$ and $V_{i}$ that is into $V_{i}$ and every node on the path is a parent of $V_{j}$. Otherwise, $V_{i}\rightarrow V_{j}$ is invisible.
\end{definition}

The visible edge is a critical concept in a MAG~\cite{zhang2008causal,perkovic2018complete}. A DAG over measured and unmeasured variables can be mapped to a MAG with measured variables. From a DAG over $\mathbf{X}\cup\mathbf{U}$ where $\mathbf{X}$ is a set of measured variables and $\mathbf{U}$ is a set of unmeasured variables, following the construction rule specified in~\cite{zhang2008completeness}, one can construct a MAG with nodes $\mathbf{X}$ such that all the conditional independence relationships among the measured variables entailed by the DAG are entailed by the MAG and vice versa, and the ancestral relationships in the DAG are maintained in the MAG. 

\subsection{Instrumental Variable (IV)}
\label{subsec:IVmethods}
Let $W$ be the treatment variable, $Y$ the outcome, and $\mathbf{X}$ be the set of all other variables. As in the literature, we consider that $\mathbf{X}$ contains pretreatment variables only, \ie for any $X\in\mathbf{X}$, $X\notin (De(W)\cup De(Y))$~\cite{abadie2003semiparametric,silva2017learning,athey2019generalized}.
When estimating the average causal effect of $W$ on $Y$, denoted as $\beta_{wy}$ from data, we follow the convention in causal inference literature, that is,  the data distribution is said to be compatible with the underlying causal DAG $\mathcal{G}$, \ie the assumptions of Markov property and faithfulness are satisfied.  

The goal of this work is to quantify the average causal effect of $W$ on $Y$, \ie $\beta_{wy}$, even when there exist unmeasured variables between $W$ and $Y$, based on observational data, by extending the existing IV techniques. 

In this section, we introduce the background information related to IVs.

\begin{definition}[Standard IV~\cite{bowden1990instrumental,hernan2006instruments}]
	\label{def:conIV}
	A variable $S$ is said to be an IV \wrt $W\rightarrow Y$, if (i) $S$ has a causal effect on $W$, (ii) $S$ affects $Y$ only through $W$ (\ie $S$ has no direct effect on $Y$), and (iii) $S$ does not share common causes with $Y$.
\end{definition}

The last two conditions of a standard IV are untestable and strict. In practice, $S$ may have other causal paths to $Y$, and $S$ is often confounded with $Y$ by other measured variables. The concept of conditional IV (CIV) in DAG is proposed to relax the conditions of a standard IV.
\begin{definition}[Conditional IV (CIV) in DAG~\cite{pearl2009causality,brito2002generalized}]
	\label{def:conditionalIV}
	Given a DAG $\mathcal{G}=(\mathbf{X}\cup \mathbf{U}\cup \{W, Y\}, \mathbf{E})$ where $\mathbf{X}$ and $\mathbf{U}$ are measured and unmeasured variables respectively. A variable $S\in \mathbf{X}$ is said to be a CIV \wrt $W\rightarrow Y$, if there exists a set of measured variables $\mathbf{Z}\subseteq\mathbf{X}$ such that (i) $S\nVbar_{d} W\mid\mathbf{Z}$, (ii) $S\Vbar_{d} Y\mid\mathbf{Z}$ in $\mathcal{G}_{\underline{W}}$, and (iii) $\forall Z\in \mathbf{Z}$, $Z\notin De(Y)$ where $\mathcal{G}_{\underline{W}}$ is obtained by removing $W\rightarrow Y$ from $\mathcal{G}$. 
\end{definition}

For a CIV $S$ as defined in Definition~\ref{def:conditionalIV}, $\mathbf{Z}$ is known to instrumentalise $S$ in the given DAG. However, a variable may be a CIV when it has zero causal effect on $W$, and this might result in a misleading conclusion. To mitigate this issue, Ancestral IV (AIV) in DAG is proposed.

\begin{definition}[Ancestral IV (AIV) in DAG~\cite{van2015efficiently}]
	\label{def:ancestralIV}
	Given a DAG $\mathcal{G}=(\mathbf{X}\cup \mathbf{U}\cup \{W, Y\}, \mathbf{E})$ where $\mathbf{X}$ and $\mathbf{U}$ are measured and unmeasured variables respectively. A variable $S\in \mathbf{X}$ is said to be an AIV \wrt $W\rightarrow Y$, if there exists a set of measured variables $\mathbf{Z}\subseteq\mathbf{X}\setminus \{S\}$ such that (i) $S\nVbar_{d} W\mid\mathbf{Z}$, (ii) $S\Vbar_{d} Y\mid\mathbf{Z}$ in $\mathcal{G}_{\underline{W}}$, and (iii) $\mathbf{Z}$ consists of $An(Y)$ or $An(S)$ or both and $\forall Z\in \mathbf{Z}$, $Z\notin De(Y)$.
\end{definition}

An AIV in DAG is a CIV in DAG, but a CIV may not be an AIV. AIV is a restricted version of CIV~\cite{van2015efficiently}. However, the applications of standard IV, CIV and AIV are established in a complete causal DAG $\mathcal{G}$, which greatly limits their capacity in real-world applications. 

Recently, Cheng \etal proposed the concept of AIV in MAG and the theorem for identifying a conditioning set $\mathbf{Z}$ that instrumentalises a given AIV in PAG~\cite{cheng2022ancestral}. However, the work by Cheng \etal assume an AIV in MAG has been given, and the focus is on finding the conditioning set for the AIV in MAG from data. Our work in this paper aims to find an AIV in MAG from data, which makes it possible for complete data-driven search for AIVs. 
\begin{proposition}[AIV in MAG~\cite{cheng2022ancestral}]
	\label{pro:pro001}
	Given a DAG $\mathcal{G} = (\mathbf{X}\cup\mathbf{U}\cup\{W, Y\}, \mathbf{E}')$ with the edges $W\rightarrow Y$ and $W\leftarrow U\rightarrow Y$ in $\mathbf{E}'$, and $U\in\mathbf{U}$, and the MAG $\mathcal{M} = (\mathbf{X}\cup\{W, Y\}, \mathbf{E})$ is mapped from $\mathcal{G}$ based on the construction rules~\cite{zhang2008completeness}. Then if $S$ is an AIV conditioning on a set of measured variables $\mathbf{Z}\subseteq\mathbf{X}\setminus\{S\}$ in $\mathcal{G}$, $S$ is an AIV conditioning on a set of measured variables $\mathbf{Z}\subseteq\mathbf{X}\setminus\{S\}$ in $\mathcal{M}$.
\end{proposition}

\begin{theorem}[Conditioning set of a given AIV in PAG~\cite{cheng2022ancestral}]
	\label{theo:IV_PAG}
	Given a DAG $\mathcal{G} = (\mathbf{X}\cup\mathbf{U}\cup\{W, Y\}, \mathbf{E}')$ with the edges $W\rightarrow Y$ and $W\leftarrow U\rightarrow Y$ in $\mathbf{E}'$, and $U\in\mathbf{U}$, and let $\mathcal{M} = (\mathbf{X}\cup\{W, Y\}, \mathbf{E})$ be the MAG mapped from $\mathcal{G}$. From data, the mapped MAG $\mathcal{M}$ is represented by a PAG $\mathcal{P}=(\mathbf{X}\cup\{W, Y\}, \mathbf{E}'')$. For a given ancestral IV $S$ which is a cause or spouse of $W$, the set $PossAn(S\cup Y)\setminus\{W, S\}$ in the learned $\mathcal{P}$ is a set that instumentalises $S$ in the DAG $\mathcal{G}$.
\end{theorem}

It is worth noting that Theorem~\ref{theo:IV_PAG} and the work in~\cite{cheng2022ancestral} are to discover conditioning set for a given AIV, rather than discovering AIVs and the conditioning set simultaneously. Hence, discovering an AIV and its conditioning set simultaneously from data remains unresolved, and it is the problem to be addressed in this work. 

\section{Discovering AIVs based on Graphical Causal Modeling}
\label{sec:algorithm}
In this section, we first introduce the generalised tetrad condition. Next, we propose a set of candidates AIVs in MAG. Then, we propose a theorem to guarantee that the generalised tetrad condition can be used to discover valid AIVs from data if there exists a pair of   AIVs. Finally, we develop a practical data-driven algorithm for estimating $\beta_{wy}$ from data.

\subsection{The Generalised tetrad Condition}
\label{subsec:tetrad}
Let $S_i$ and $S_j$ be a pair of CIVs given the conditioning set $\mathbf{X}\setminus\{S_i, S_j\}$. Let $\sigma_{s_{i}*y*\mathbf{z}}$  $(\sigma_{s_{j}*y*\mathbf{z}})$ denote the partial covariance of $S_i$ ($S_j$) and $Y$ given $\mathbf{Z}$, and $\sigma_{s_{i}*w*\mathbf{z}}$ $(\sigma_{s_{j}*w*\mathbf{z}})$ denote the partial covariance of $S_i$ ($S_j$) and $W$ given $\mathbf{Z}$. Then, we have $\beta_{wy}= \sigma_{s_{i}*y*\mathbf{z}}/\sigma_{s_{i}*w*\mathbf{z}} = \sigma_{s_{j}*y*\mathbf{z}}/\sigma_{s_{j}*w*\mathbf{z}}$, which gives us the following tetrad condition~\cite{silva2017learning}:
\begin{equation}
	\label{eq:tetradCon}
	\sigma_{s_{i}*y*\mathbf{z}}\sigma_{s_{j}*w*\mathbf{z}} - \sigma_{s_{i}*w*\mathbf{z}}\sigma_{s_{j}*y*\mathbf{z}} = 0
\end{equation}

The tetrad condition can be tested from data directly. It is a necessary condition for discovering valid CIVs, which means a pair of variables that are not valid CIVs can also satisfy the tetrad condition. 

We consider a more general setting, where a pair of AIVs $S_i$ and $S_j$ have different conditioning sets $\mathbf{Z}_i \subseteq \mathbf{X}\setminus\{S_i\}$ and $\mathbf{Z}_j \subseteq \mathbf{X}\setminus\{S_j\}$ respectively, and $\mathbf{Z}_i$ and $\mathbf{Z}_j$ do not need to be equal. Let $\sigma_{s_{i}*y*\mathbf{z}_i}$ $(\sigma_{s_{j}*y*\mathbf{z}_j})$ denote the partial covariance of $S_i$ ($S_j$) and $Y$ given $\mathbf{Z}_i$ ($\mathbf{Z}_j$), and $\sigma_{s_{i}*w*\mathbf{z}_i}$ ($\sigma_{s_{j}*w*\mathbf{z}_j}$) denote the partial covariance of $S_i$ ($S_j$) and $W$ given $\mathbf{Z}_i$ ($\mathbf{Z}_j$). Then we have  $\beta_{wy}= \sigma_{s_{i}*y*\mathbf{z}_i}/\sigma_{s_{i}*w*\mathbf{z}_i} = \sigma_{s_{j}*y*\mathbf{z}_j}/\sigma_{s_{j}*w*\mathbf{z}_j}$, which gives us the following generalised tetrad condition:	
\begin{equation}
	\label{eq:ourtetradCon}
	\sigma_{s_{i}*y*\mathbf{z}_i}\sigma_{s_{j}*w*\mathbf{z}_j} - \sigma_{s_{i}*w*\mathbf{z}_i}\sigma_{s_{j}*y*\mathbf{z}_j} = 0
\end{equation}

In the following, we will show that the generalised tetrad condition can be used for finding a pair of AIVs in data if there exists a pair of AIVs. In comparison with the tetrad condition used in IV.tetrad~\cite{silva2017learning}, the search space of the generalised tetrad condition is larger since $S_i, S_j, \mathbf{Z}_i$ and $\mathbf{Z}_j$ all vary. In the next section, we will develop a lemma to reduce the search space.

\subsection{The Theory for Discovering AIVs in MAG}
\label{Subsec:relaxedIV}
We aim to develop a practical solution for discovering AIVs directly from data by leveraging the property of a MAG.  

We first categorise AIVs into direct AIVs and indirect AIVs. When $S$ is an AIV and it is an adjacent node of the treatment in the given DAG, its ancestral or adjacent nodes may be AIVs. We call $S$ a direct AIV and the AIVs which are ancestral or adjacent nodes of $S$ indirect AIVs if they satisfy Definition~\ref{def:ancestralIV}. We consider direct AIVs since, in practice, indirect AIVs are rare. Importantly, direct AIVs have a property to support the data-driven search for their conditioning sets. In the following discussions, all AIVs are direct AIVs. We have the following conclusion for discovering the direct AIVs in a MAG $\mathcal{M}$.  

\begin{lemma}[A direct AIV in MAG]
	\label{lemma:possAnIV}
Given a DAG $\mathcal{G} = (\mathbf{X}\cup\mathbf{U}\cup\{W, Y\}, \mathbf{E}')$ with the edges $W\rightarrow Y$ and $W\leftarrow U\rightarrow Y$ in $\mathbf{E}'$, and $U\in\mathbf{U}$, and let $\mathcal{M} = (\mathbf{X}\cup\{W, Y\}, \mathbf{E})$ be the MAG mapped from $\mathcal{G}$. If $S$ is a direct AIV in the DAG $\mathcal{G}$, then $S\in Adj(Y)\setminus\{W\}$ in MAG $\mathcal{M}$.
\end{lemma}
\begin{proof}
Firstly, the edges $W\rightarrow Y$ and $W\leftarrow U\rightarrow Y$ in the given DAG are represent by an invisible edge $W\rightarrow Y$ in the mapped MAG $\mathcal{M}$~\cite{zhang2008completeness}.	For an $S \in \mathbf{X}$ to be an eligible direct AIV in the DAG $\mathcal{G}$, there are only two cases in the mapped MAG $\mathcal{M}$. The first case is that $S$ has an edge $S\astrightarrow W$, then $S$ must have an edge into $Y$, \ie $S\in Adj(Y)\setminus\{W\}$, since otherwise $W\rightarrow Y$ in $\mathcal{M}$ is visible, which contradicts the invisible edge $W\rightarrow Y$ in $\mathcal{M}$. The second case is that $S$ has a collider path into $W$ and every collider on the path is in $Pa(Y)$, \ie $S\nVbar W\mid Pa(Y)\setminus\{W\}$, then $S$ must have an edge into $Y$, \ie $S\in Adj(Y)\setminus\{W\}$, since otherwise $W\rightarrow Y$ in $\mathcal{M}$ is visible, which contradicts the invisible edge $W\rightarrow Y$ in $\mathcal{M}$. Therefore, all direct AIVs in the DAG $\mathcal{G}$ are included in the set $\mathbf{S}=Adj(Y)\setminus\{Y\}$ of the mapped MAG $\mathcal{M}$.
\end{proof}

Lemma~\ref{lemma:possAnIV} provides a set of candidate direct AIVs $Adj(Y)\setminus\{W\}$ and reduces the search space of a direct AIV, \ie the search space of a direct AIV is reduced from $\mathbf{O}({|\mathbf{X}|})$ to $\mathbf{O}({|Adj(Y)\setminus\{W\}|})$ where $|Adj(Y)\setminus\{W\}|\ll |\mathbf{X}|$. 

Next, we will develop a theorem to show that the generalised tetrad condition in Eq.(\ref{eq:ourtetradCon}) can be used to discover a pair of direct AIVs  directly from data with latent variables. 
\begin{theorem}
\label{theo:tetrad}
Given a DAG $\mathcal{G} = (\mathbf{X}\cup\mathbf{U}\cup\{W, Y\}, \mathbf{E}')$ with the edges $W\rightarrow Y$ and $W\leftarrow U\rightarrow Y$ in $\mathbf{E}'$, and $U\in\mathbf{U}$, and let $\mathcal{M} = (\mathbf{X}\cup\{W, Y\}, \mathbf{E})$ be the MAG mapped from $\mathcal{G}$. Let  $\mathcal{P}=(\mathbf{X}\cup\{W, Y\}, \mathbf{E}'')$ be the PAG which encodes the set of MAGs Markov equivalent to $\mathcal{M}$. If there exists a pair of direct AIVs $\{S_i, S_j\} \subseteq\mathbf{X}$ in the DAG $\mathcal{G}$, then $\{S_i, S_j\}$ must be in $Adj(Y)\setminus\{W\}$ in the PAG $\mathcal{P}$. Moreover, the two sets, $\mathbf{Z}_i = possAn(S_i \cup Y)\setminus\{W, S_i\}$ and $\mathbf{Z}_j = possAn(S_j \cup Y)\setminus\{W, S_j\}$ in the PAG $\mathcal{P}$ instrumentalise $S_i$ and $S_j$ in the DAG $\mathcal{G}$, respectively. Hence, Eq.(\ref{eq:ourtetradCon}) holds for $S_i$ and $S_j$ and their conditioning sets $\mathbf{Z}_i$ and $\mathbf{Z}_j$.
\end{theorem} 
\begin{proof}
According to Lemma~\ref{lemma:possAnIV}, $Adj(Y)\setminus\{W\}$ in the mapped MAG $\mathcal{M}$ is the set of candidate direct AIVs in the DAG $\mathcal{G}$. Hence, the set $Adj(Y)\setminus\{W\}$ in the PAG $\mathcal{P}$ must be the set of candidate direct AIVs because the mapped MAG $\mathcal{M}$ is encoded in the PAG $\mathcal{P}$. Thus, if $S\in \mathbf{X}$ is a direct AIV in the DAG $\mathcal{G}$, then $S \in Adj(Y)\setminus\{W\}$ in the PAG $\mathcal{P}$ \ie $\{S_i, S_j\} \subseteq \mathbf{S}$ holds. According to Theorem~\ref{theo:IV_PAG}, $\mathbf{Z}_i = possAn(S_i \cup Y)\setminus\{W, S_i\}$ and $\mathbf{Z}_j = possAn(S_j \cup Y)\setminus\{W, S_j\}$ in the PAG $\mathcal{P}$ instrumentalise $S_i$ and $S_j$ in the DAG $\mathcal{G}$, respectively. Thus, we have $\beta_{wy} = \sigma_{s_{i}*y*\mathbf{z}_i}/\sigma_{s_{i}*w*\mathbf{z}_i} = \sigma_{s_{j}*y*\mathbf{z}_j}/\sigma_{s_{j}*w*\mathbf{z}_j}$. Therefore, $\sigma_{s_{i}*y*\mathbf{z}_i}\sigma_{s_{j}*w*\mathbf{z}_j} - \sigma_{s_{i}*w*\mathbf{z}_i}\sigma_{s_{j}*y*\mathbf{z}_j} = 0$, \ie Eq.(\ref{eq:ourtetradCon}) holds.
\end{proof}

Theorem~\ref{theo:tetrad} supports a data-driven algorithm to discover a pair of direct AIVs $\{S_i, S_j\}$ and their corresponding conditioning sets $\mathbf{Z}_i$ and $\mathbf{Z}_j$ by utilising the generalised tetrad condition. In the next section, based on the theorem, we will propose a practical algorithm for estimating $\beta_{wy}$ from data with latent variables.

Note that a significant number of direct AIVs are in both $Adj(W)\setminus\{Y\}$ and $Adj(Y)\setminus\{W\}$ in a MAG. Sometimes, they may be missed from $Adj(Y)\setminus\{W\}$ due to the false discoveries of the structure learning algorithm used~\cite{aliferis2010local,kalisch2012causal}.  In the corresponding DAG, the direct AIVs are closer to $W$ than $Y$. To avoid the random fluctuations without sacrificing much efficiency, in our developed practical algorithm, we extend the search space of Lemma~\ref{lemma:possAnIV} to $Adj(W\cup Y)\setminus\{W, Y\}$. This only adds minor additional costs to the search process.

\subsection{A Practical Algorithm for Estimating $\beta_{wy}$}
\label{subsec:datadrimeth}
We develop a practical data-driven algorithm, AIV.GT (\underline{A}ncestral \underline{IV} based on \underline{G}eneralised \underline{T}etrad condition), for estimating $\beta_{wy}$ from data with latent variables. The pseudocode of AIV.GT is listed in Algorithm~\ref{alg:AnIVsDatadd}. 

\begin{algorithm}[t]
	\caption{AIVs based on the Generalised Tetrad condition (AIV.GT)}
	\label{alg:AnIVsDatadd}
	\begin{algorithmic}[1]
		\STATE {\bfseries Input:} The set of pretreatment variables $\mathbf{X}$, the treatment $W$, outcome $Y$ and the dataset $\mathcal{D}$; $\alpha$ =0.05
		\STATE {\bfseries Output:} $\hat{\beta}_{wy}$, the causal effect of $W$ on $Y$, or \emph{NA}, i.e. lacking knowledge 
		\STATE Recover a PAG $\mathcal{P}$ from  $\mathcal{D}$ by using the \textit{rfci} algorithm
		\STATE Obtain $\mathbf{S}=Adj(W\cup Y)\setminus\{W, Y\}$ from $\mathcal{P}$
		\IF{$\left |\mathbf{S}\right|\leqslant 1$}
		\STATE {\bfseries return} \emph{NA}
		\ELSE
		\FOR{each $S_i\in \mathbf{S}$}
		\STATE $\mathbf{Z}_i \leftarrow PossAn(S_i\cup Y)\setminus\{S_i,W,Y\}$
		\STATE $\hat{\beta}_{i} \leftarrow TSLS(W, Y, S_i, \mathbf{Z}_i, \mathcal{D})$ 
		\ENDFOR
		\STATE Initialise $\mathbf{Q}=\emptyset$
		\FOR{each pair $(S_i, S_j) \in \mathbf{S}$}
		\IF{$Test.tetrad(W, Y, S_i, S_j, \mathbf{Z}_i, \mathbf{Z}_j, \mathcal{D}, \alpha)$}
		\STATE $\epsilon_{ij} = \left|\sigma_{s_{i}*y*\mathbf{z}_i}\sigma_{s_{j}*w*\mathbf{z}_j}- \sigma_{s_{i}*w*\mathbf{z}_i}\sigma_{s_{j}*y*\mathbf{z}_j} \right|$
		\STATE $\lambda_{ij} =  \left |\epsilon_{ij} -\delta_{ij} \right|$ where $\delta_{ij} = \left | \hat{\beta}_{i} -\hat{\beta}_{j} \right|$ 
		\STATE $\mathbf{Q} \leftarrow \mathbf{Q}\cup\lambda_{ij}$
		\ENDIF
		\ENDFOR
		\IF {$\left|\mathbf{Q}\right|$ = $\emptyset$}
		\STATE {{\bfseries return} \emph{NA}}
		\ELSE
		\STATE {\bfseries return} $\hat{\beta}_{wy} = mean(\hat{\beta}_{i}, \hat{\beta}_{j})$ \\ where the consistent score $\lambda_{ij}$ is the smallest in $\mathbf{Q}$
		\ENDIF
		\ENDIF
	\end{algorithmic}
\end{algorithm}

AIV.GT aims to search for the pair of   AIVs from data directly without domain knowledge. The generalised tetrad condition in Eq.(\ref{eq:ourtetradCon}) held by a pair of   AIVs and their conditioning sets as described in Theorem~\ref{theo:tetrad} if there is a pair of   IVs $\{S_i, S_j\}$ in data.

To obtain a reliable result, we propose a consistency score to assess which paired variables are the most likely AIVs based on the generalised tetrad condition. Let $\epsilon_{ij} = \left |\sigma_{s_{i}*y*\mathbf{z}_i}\sigma_{s_{j}*w*\mathbf{z}_j} - \sigma_{s_{i}*w*\mathbf{z}_i}\sigma_{s_{j}*y*\mathbf{z}_j} \right|$, and $\delta_{ij} = \left | \hat{\beta}_{i} -\hat{\beta}_{j} \right|$ where $\hat{\beta}_{i}$ and $\hat{\beta}_{j}$ are the estimated causal effects of $W$ on $Y$ by using $S_i$ and $S_j$ as an instrument, respectively. The consistency score is defined as $\lambda_{ij} =  \left |\epsilon_{ij} -\delta_{ij} \right|$.

The justification of the consistency score is that $\epsilon_{ij}$ is expected to be close to 0, and the same with $\delta_{ij}$ 0 if the variables $S_i$ and $S_j$ are AIVs. Theoretically, the pair of variables with either the smallest $\epsilon_{ij}$ or $\delta_{ij}$ are most likely to be the pair of   IVs, but in practical cases, a pair of variables passing the generalised tetrad condition test (\ie a small enough $\epsilon_{ij}$), may have a large $\delta_{ij}$, or vice versa, since the pair of variables are not   IVs. So we use their difference $\lambda_{ij}$ to avoid such cases because $\lambda_{ij}$ must be smaller than both $\epsilon_{ij}$ and $\delta_{ij}$. Under the assumption that there exists at least a pair of   IVs,  if the consistency score of a pair of variables is the smallest, then the pair are most likely to be AIVs. The paired variables with the minimal consistency score is returned as the result of AIV.GT. 

The AIV.GT algorithm is divided into two parts. The first part (Lines 3 to 11) is to obtain all candidate AIV pairs and the possible causal effects of $W$ on $Y$ estimated using these pairs. Line 3 aims to learn a PAG $\mathcal{P}$ from data by using a causal structure learning algorithm~\cite{scutari2009learning,kalisch2012causal}. We use \textit{rfci} (really fast causal inference)~\cite{colombo2012learning} in AIV.GT. Line 4 aims to get the set of candidate AIVs from the learned PAG $\mathcal{P}$. Line 5 tests the size of $\mathbf{S}$ and if $\left |\mathbf{S}\right|\leqslant 1$, then AIV.GT returns \textit{NA} due to lack of knowledge. Lines 8 to 11 are to estimate the causal effect $\hat{\beta}_{wy}$ using each candidate AIV. Line 9 is to find the conditioning set $\mathbf{Z}_i$ for a candidate AIV $S_i$ based on Theorem~\ref{theo:IV_PAG}. Line 10, the function $TSLS()$ is the estimator of two-stage least squares (TSLS) by using $S_i$ as an IV and conditioning on $\mathbf{Z}_i$ for calculating $\hat{\beta}_{i}$. 

The second part of AIV.GT is to discover the pair of   AIVs. Line 12 is to initialise the set of consistency scores $\mathbf{Q}$. Lines 13 to 19 are to check the validity of each pair of candidate AIVs based on Theorem~\ref{theo:tetrad}. If the generalised tetrad condition holds on a pair of candidate AIVs, then calculate their consistency score. Line 14, the function $Test.tetrad()$ is implemented by using the \textit{Wishart test} \wrt the generalised tetrad condition~\cite{wishart1928sampling,spirtes2000causation}. $Test.tetrad()$ returns \textit{TRUE} if and only if the set of candidate pair variables returns a p-value greater than the significant level $\alpha$ ($\alpha$ =0.05 in this work). Lines 15 to 17 are to obtain the consistency score of each paired AIVs satisfying the generalised tetrad condition. In Lines 20 and 21, if $\mathbf{Q}$ is an empty set, then no pair of variables has passed the tetrad condition test and the algorithm returns \textit{NA}. In Lines 22 to 24, AIV.GT returns the mean causal effect of the pair of variables with the smallest $\lambda_{ij}$ in $\mathbf{Q}$. 

\noindent \textbf{Time Complexity Analysis:} Three factors contribute to the time complexity of AIV.GT. The first contributing factor is the learning of a PAG $\mathcal{P}$ from data and finding $\mathbf{S}$ from $\mathcal{P}$, which largely relies on the \textit{rfci} algorithm. In the worst situation, \textit{rfci} has a complexity of $\mathbf{O}(2^{r}*n)$, where $r$ is the maximum degree of a node in the underlying causal MAG and $n$ is the sample size. In most cases, the average degree of a causal Bayesian network is 2 to 5~\cite{scutari2009learning}, and most of the underlying MAGs are sparse in real-world applications. Hence, the time complexity of \textit{rfci} is lower~\cite{colombo2012learning}. The complexity of Line 4 is $\mathbf{O}(1)$ since it reads from $\mathcal{P}$. The second factor is estimating all possible causal effects, \ie Lines 8 to 11 in Algorithm~\ref{alg:AnIVsDatadd}. Noting that obtaining $\mathbf{Z}_i$ takes $\mathbf{O}(1)$ and calculating $\hat{\beta}_i$ needs $\mathbf{O}(n*p^{2})$. Hence, the whole time complexity of this part is $\mathbf{O}(|\mathbf{S}|*n*p^{2})$. The third factor is finding the pair of   IVs from $\mathbf{S}$ and time complexity relies on the size of $\mathbf{S}$ (pairwise search for a pair of VIs) and calculating covariance, which all together takes $\mathbf{O}(|\mathbf{S}|^{2}*n*p^{2})$. Therefore, the overall complexity of AIV.GT is $\mathbf{O}(2^{r}*n + |\mathbf{S}|*n*p^{2} +|\mathbf{S}|^{2}*n*p^{2}) =\mathbf{O}(2^{r}*n +|\mathbf{S}|^{2}*n*p^{2})$. Therefore, the complexity of AIV.GT is largely attribute to the \textit{rfci} algorithm and searching for a pair satisfying the generalised tetrad criterion. 

\section{Experiments}
\label{sec:exp}
We assess the performance of AIV.GT by comparing it to the state-of-the-art causal effect estimators, firstly with a simulation study. Then, we conduct experiments on two real-world datasets that have been used for a long time in instrumental variable research~\cite{card1993using,sjolander2019instrumental} to show that AIV.GT can be applied in real-world applications. 

\begin{figure*}[ht]
	\vskip -0.1in
	\begin{center}
		\centerline{\includegraphics[scale=0.59]{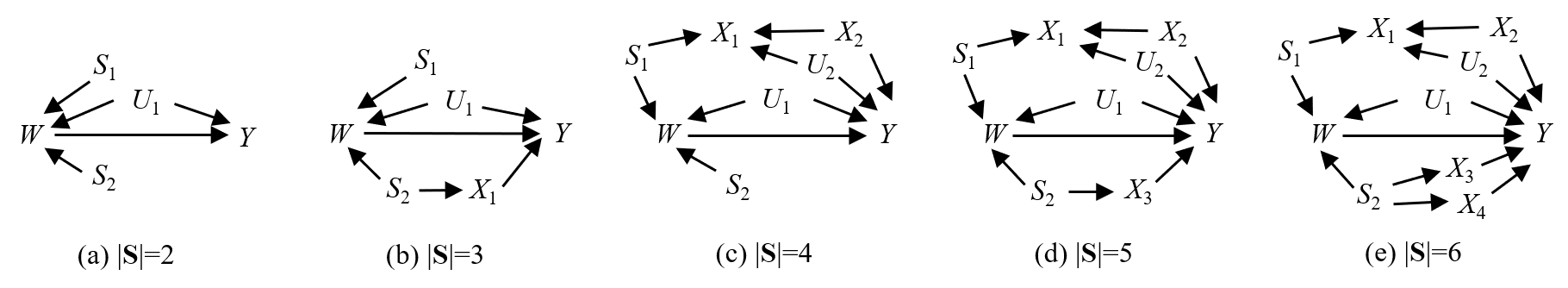}}\vskip -0.1in
		\caption{The core of the true causal DAGs over measured and unmeasured used to generate synthetic datasets. $\left|\mathbf{S}\right|$ denotes the number of measured variables $Adj(W\cup Y)\setminus\{W, Y\}$. There are two valid IVs $S_1$ and $S_2$ in all causal DAGs. In (a), $S_1$ and $S_2$ are standard IVs. In (b), $S_1$ is a standard IV and $S_2$ is a CIV conditioning on $X_1$. In (c), $S_1$ is a CIV conditioning on $\emptyset$ since $X_1$ is a collider and $S_2$ is a standard IV. In (d), $S_1$ and $S_2$ are CIVs conditioning on $\emptyset$ and $X_3$, respectively. In (e), $S_1$ is a condition IV conditioning on $\emptyset$ and $S_2$ is a CIV conditioning on $\{X_3, X_4\}$.}
		\label{fig:theDAGs_syn_datasets}
	\end{center}
	\vskip -0.1in
\end{figure*}

The estimators compared include (1) LSR, least squares regression $Y$ on $\{W, \mathbf{X}\}$; (2) TSLS, two-stage least squares (TSLS~\cite{angrist1995two}) of $Y$ on $W$ using all variables $\mathbf{X}$ as standard IVs; (3) some invalid some valid IV estimator (sisVIVE)~\cite{kang2016instrumental}; (4) IV.tetrad method~\cite{silva2017learning}. LSR is not an IV based method. It is included since it is frequently used in Machine Learning disregarding bias of latent variables in data, and it is used as a baseline. All other three comparison estimators do not need a nominated IV. TSLS is a standard IV estimator, and it is also used as a baseline. sisVIVE and IV.tetrad are two most related methods and have been discussed in the Introduction.  

\noindent \textbf{The Implementations of Estimators in Sections~\ref{subsec:simulation} and~\ref{subsubsec:withoutIV}.} The method OLS is implemented by the function \emph{cov} in the $\mathbf{R}$ package \emph{stats}. The TSLS is programmed by the functions \emph{cov} in the $\mathbf{R}$ package \emph{stats} and \emph{solve} in the base. The implementation of LSR is same with TSLS, \ie using the functions \emph{cov} and \emph{solve}. The implementation of sisVIVE is based on the function \emph{sisVIVE} in the $\mathbf{R}$ package \emph{sisVIVE}.	The implementation of IV.tetrad is retrieved from the authors' site\footnote{\url{http://www.homepages.ucl.ac.uk/~ucgtrbd/code/iv_discovery}}. The parameter of \emph{num$\_$ivs} is set to 3 (2 for VitD). 
AIV.GT is implemented by using the function \emph{rfci} in the $\mathbf{R}$ packages \emph{pcalg}, \emph{cov} in the $\mathbf{R}$ package \emph{stats}, \emph{solve} in the base and the functions in IV.tetrad. 

\noindent \textbf{The Implementations of Estimators in Section~\ref{subsubsec:withIV}.} The implementations of the compared estimators that require a known IV are introduced as follows.  The estimator TSLS is implemented by the function \emph{ivreg} in the $\mathbf{R}$ package \emph{AER}~\cite{greene2003econometric}. The implementation of TSLS.CIV is based on the functions \emph{glm} and \emph{ivglm} from the $\mathbf{R}$ packages \emph{stats} and \emph{ivtools}~\cite{sjolander2019instrumental}. FIVR is implemented by the function \emph{instrumental}$\_$\emph{forest} in the $\mathbf{R}$ package \emph{grf}~\cite{athey2019generalized}. All parameters in FIVR are default. AIViP is implemented by the functions \emph{rfci} in $\mathbf{R}$ package \emph{pcalg}~\cite{kalisch2012causal}, \emph{glm} in $\mathbf{R}$ package \emph{stats} and \emph{ivglm} in $\mathbf{R}$ package \emph{ivtools}~\cite{sjolander2019instrumental}. 

\noindent \textbf{Evaluation Metrics \& Parameter Setting.} For the simulation study, we have the ground truth of $\beta_{wy}$, so we report the estimation bias: $\left |(\hat{\beta}_{wy}-\beta_{wy})/\beta_{wy} \right|*100$ (\%). In the experiments with real-world datasets, we empirically evaluate all estimators, and then compare AIV.GT with four additional IV-based estimators that require a nominated IV. The significant level $\alpha$ is set to 0.05 for the functions of \textit{rfci} and $Test.tetrad()$ in all experiments.

\subsection{Simulation Study}
\label{subsec:simulation}
The goal of this set of experiments is to test the effectiveness of AIV.GT with  and without colliders in the covariate set in comparison with various causal effect estimators. We utilise five true DAGs over $\mathbf{X}\cup\mathbf{U}\cup\{W, Y\}$ to generate five synthetic datasets with latent variables. The five true DAGs are shown in Fig.~\ref{fig:theDAGs_syn_datasets}. In addition to the variables in the five true DAGs, 20 additional measured variables for each dataset are generated as noise variables that are related to each other but not to the variables in the five DAGs. The additional 20 variables are generated by a multivariate normal distribution. Next, we will separately introduce the details of each synthetic dataset generation based on the five true DAGs in Fig.~\ref{fig:theDAGs_syn_datasets}.   

The synthetic dataset (a) is generated from the DAG (a) in Fig.~\ref{fig:theDAGs_syn_datasets}, and the specifications are as following: $U_1\sim Bernoulli(0.5)$ and $S_1, S_2 \sim N(0, 1)$, in which $N(,)$ denotes the normal distribution. The treatment $W$ is generated from $n$ ($n$ denotes the sample size) Bernoulli trials by using the assignment probability $P(W=1\mid U_1, S_1, S_2) = [1+exp\{1-3*U_1-3*S_1-3*S_2\}]$. The potential outcome is generated from $Y_{W} = 2 + 2*W + 3*U_1 +\epsilon_{w}$ where $\epsilon_{w}\sim N(0,1)$.

The synthetic dataset (b) is generated from the DAG (b) in Fig.~\ref{fig:theDAGs_syn_datasets}, and the specifications are as following: $U_1\sim Bernoulli(0.5)$,  $S_1, S_2 \sim N(0, 1)$ and $\epsilon_{X_1} \sim N(0, 0.5)$.  The measured variable $X_1 = 0.8*S_2 + \epsilon_{X_1}$.The treatment $W$ is generated by $P(W=1\mid U_1, S_1, S_2) = [1+exp\{1-3*U_1-3*S_1-3*S_2\}]$. The potential outcome is generated from $Y_{W} = 2 + 2*W + 3*U_1 + 3*X_1  +\epsilon_{w}$.

The synthetic dataset (c) is generated from the DAG (c) in Fig.~\ref{fig:theDAGs_syn_datasets}, and the specifications are as following: $U_1\sim Bernoulli(0.5)$, $S_1, S_2, U_2, X_2 \sim N(0, 1)$ and $\epsilon_{X_1} \sim N(0, 0.5)$. The measured variable $X_1$ is generated by $X_1 = 0.3 + S_1 + X_2 + U_2 + \epsilon_{X_1}$. The treatment $W$ is generated by $P(W=1\mid U_1, S_1, S_2) = [1+exp\{1-3*U_1-3*S_1-3*S_2\}]$. The potential outcome is generated from $Y_{W} = 2 + 2*W + 3*U_1 + 2*U_2 + 2*X_2 +\epsilon_{w}$.

The synthetic dataset (d) is generated from the DAG (d) in Fig.~\ref{fig:theDAGs_syn_datasets}, and the specifications are as following: $U_1\sim Bernoulli(0.5)$, $S_1, S_2, U_2, X_2 \sim N(0, 1)$ and $\epsilon_{X_1}, \epsilon_{X_3}\sim N(0, 0.5)$. The measured variables $X_1$ and $X_3$ are generated by $X_1 = 0.3+  S_1 + X_2 + 1.5*U_2 +\epsilon_{X_1}$ and $X_3 = 0.8*S_2 +\epsilon_{X_3}$, respectively. The treatment $W$ is generated by  $P(W=1\mid U_1, S_1, S_2) = [1+exp\{1-3*U_1-3*S_1-3*S_2\}]$. The potential outcome is generated by $Y_{W} = 2 + 2*W + 3*U_1 + 2*U_2 + 2*X_2 +2*X_3 +\epsilon_{w}$.

The synthetic dataset (e) is generated from the DAG (e) in Fig.~\ref{fig:theDAGs_syn_datasets}, and the specifications are as following: $U_1\sim Bernoulli(0.5)$, $S_1, S_2, U_2, X_2 \sim N(0, 1)$ and $\epsilon_{X_1}, \epsilon_{X_3}, \epsilon_{X_4} \sim N(0, 0.5)$. The measured variables $X_1$, $X_3$ and $X_4$ are generated by $X_1 = 0.3+  S_1 + X_2 + 1.5*U_2 +\epsilon_{X_1}$,  $X_3 = 0.8*S_2 +\epsilon_{X_3}$ and  $X_4 = 0.8*S_2 +\epsilon_{X_4}$, respectively. The treatment $W$ is generated by $P(W=1\mid U_1, S_1, S_2) = [1+exp\{1-3*U_1-3*S_1-3*S_2\}]$. The potential outcome is generated by $Y_{W} = 2 + 2*W + 3*U_1 + 2*U_2 + 2*X_2 +2*X_3 +2*X_4 +\epsilon_{w}$.

The suitability of datasets with the method requirements is summarised as Table~\ref{tab_suitability}.

\begin{table}[ht]
	\centering
	\caption{Satisfaction (tick)/violation (cross) of the assumptions of a method by a dataset. ``$\checkmark$ ?'' means the problem of collider bias suffered by IV.tetrad.}
	\begin{tabular}{|c|c c c c c|}
		\hline 
		& (a)  & (b) & (c) & (d) & (e) \\
		\hline
		LSR & $\times$ & $\times$ & $\times$ & $\times$ & $\times$ \\
		TSLS & $\checkmark$ & $\times$ & $\times$ & $\times$ & $\times$ \\
		sisVIVE & $\checkmark$ & $\checkmark$ & $\checkmark$ & $\times$ & $\times$ \\
		IV.tetrad & $\checkmark$ & $\checkmark$ & $\checkmark$ ? & $\checkmark$ ? & $\checkmark$ ? \\
		AIV.GT & $\checkmark$ & $\checkmark$ & $\checkmark$ & $\checkmark$ & $\checkmark$ \\
		\hline
	\end{tabular}
	\label{tab_suitability}
\end{table}

\begin{figure}[t]
	\vskip -0.02in
	\begin{center}
		\centerline{\includegraphics[scale=0.4]{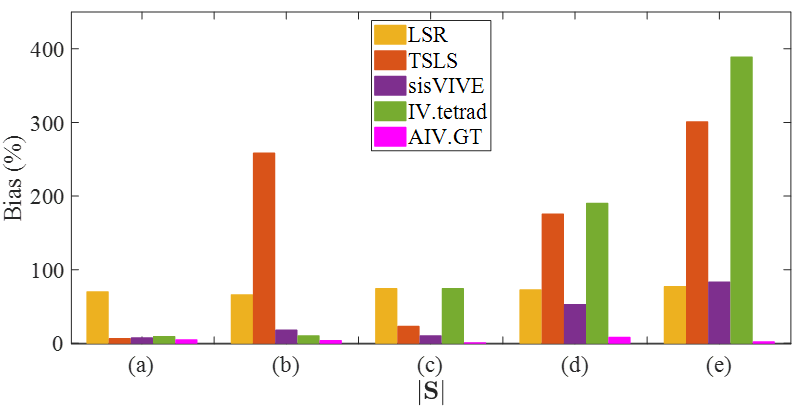}}
		\vskip -0.18in
		\caption{Estimation Bias (\%) of the estimators on five synthetic datasets. AIV.GT has the smallest bias on all datasets. }
		\label{fig:Syn_Bias}
	\end{center}
	\vskip -0.3in
\end{figure}

\paragraph{Results.} The estimation biases AIV.GT and the four compared estimators on the synthetic datasets are visualised in Fig.~\ref{fig:Syn_Bias}. From Fig.~\ref{fig:Syn_Bias}, we have the following observations: (1) The biases of LSR are large on all datasets. This is because it does not consider any bias of latent variables in data. (2) TSLS has a low bias on dataset (a) because $S_1$ and $S_2$ are standard IVs in this dataset. It has large biases on all other datasets since there are no standard IVs in the datasets. (3) sisVIVE works well on the first three datasets, i.e. (a) $-$ (c), since its requirement (i.e. a half of covariates are IVs) is satisfied. It works poorly on datasets (d) and (e) because of the assumption is violated. (4) IV.tetrad has a low biases on the first two datasets, but large biases on the last three datasets because the last three datasets contain colliders in the covariate sets. All datasets satisfy the assumption of IV.tetrad (i.e. a pair of   CIVs), but it suffers the problem of collider bias as identified in this paper. (5) AIV.GT obtains consistent results and has the lowest bias on all datasets. 

In sum, AIV.GT is able to obtain an unbiased $\hat{\beta}_{wy}$ from data with latent variables when there exists a pair of AIVs. AIV.GT overcomes the collider bias suffered by IV.tetrad.

\subsection{Experiments on Two Real-world Datasets}
\label{subsec:realworld}
It is challenging to evaluate the performance of causal effect estimators, including AIV.GT on real-world datasets because the ground truth causal DAG and $\beta_{wy}$ are not available. We select two real-world datasets with empirical estimates available in the literature, including Vitamin D data (VitD)~\cite{martinussen2019instrumental,sjolander2019instrumental} and Schoolingreturns~\cite{card1993using}. The two datasets have been extensively studied and analysed before and each has a nominated AIV. Therefore, it is credible to choose them as the benchmark datasets to evaluate AIV.GT.

The two datasets have the nominated AIV \wrt $(W, Y)$, but the conditioning sets are unknown. There is not an available algorithm in literature to discover the conditioning set that instrumentalise the nominated AIV on both datasets. Therefore, we divide the experiments on both datasets into two parts: (1) experiments on AIV.GT in comparison with four estimators without nominated AIVs; (2) experiments comparing AIV.GT with four additional IV estimators that require a nominated AIV. 

\subsubsection{Details of the Two Real-world Datasets}
\paragraph{Vitamin D (VitD).} This dataset was collected from a cohort study of vitamin D status on mortality, \ie the potential effect of VitD on death, reported in~\cite{martinussen2019instrumental}. The dataset contains 2571 individuals and 5 variables: age, filaggrin (a binary variable indicating filaggrin mutations), vitd (a continuous variable measured as serum 25-OH-D (nmol/L)), time (follow-up time), and death (binary outcome indicating whether an individual died during follow-up)~\cite{sjolander2019instrumental}. A measured value of vitamin D less than 30 nmol/L implies vitamin D deficiency. We take the estimated $\hat{\beta}_{wy} = 2.01$ with 95\% confidence interval $(0.96, 4.26)$ from the literature~\cite{martinussen2019instrumental} as the reference causal effect.

\paragraph{Schoolingreturns.} This dataset is from the national longitudinal survey of youth (NLSY) of US young employees, aged range from 24 to 34~\cite{card1993using}. The dataset contains 3010 individuals and 19 variables. The treatment is the education of employees, and the outcome is raw wages in 1976 (in cents per hour). The covariates include experience (years of labour market experience), ethnicity (a factor indicating ethnicity), resident information of an individual, age, nearcollege (whether an individual grew up near a 4-year college), Education in 1966 (\textit{education66}), marital status, father's educational attainment (\textit{feducation}), mother's educational attainment (\textit{meducation}), Ordered factor coding family education class (\textit{fameducation}), and so on. A goal of the studies on this dataset is to investigate the causal effect of education on earnings. We take $\hat{\beta}_{wy} = 13.29$\% with 95\% confidence interval $(0.0484, 0.2175)$ from~\cite{verbeek2008guide} as the reference causal effect.

\subsubsection{Comparing AIV.GT with the Estimators without Requiring a Known IV}
\label{subsubsec:withoutIV}
We conduct experiments on the two real-world datasets to assess AIV.GT against the four estimators that do not require nominated IV as in Section~\ref{subsec:simulation} with the simulated data. All experimental results are visualised in Fig.~\ref{fig:tworealworlddatasetswithoutIV} for VitD and Schoolingreturns, respectively. 

\begin{figure}[t]
	\vskip -0.02in
	\begin{center}
		\centerline{\includegraphics[scale=0.4]{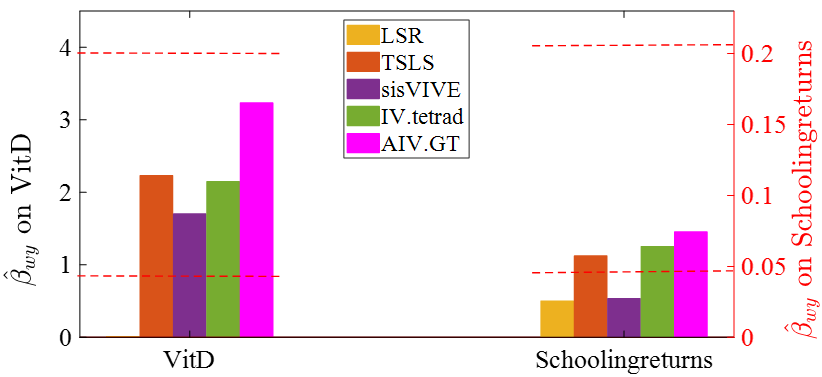}}
		\vskip -0.1in
		\caption{The experimental results of the five estimators without a given IV for all estimators on two real-world datasets. The two dotted lines represent empirically estimated causal effect with 95\% confidence interval. Noting that the estimated causal effect of LSR on VitD is close to zero and not visible in the left panel.  }
		\label{fig:tworealworlddatasetswithoutIV}
	\end{center}
	\vskip -0.2in
\end{figure}

\paragraph{Results on VitD.} From Fig.~\ref{fig:tworealworlddatasetswithoutIV}, we have the following observations: (1) the estimated result of LSR is close to 0 and far away from the 95\% confidence interval of the empirical estimation; (2) the estimated results of TSLS, sisVIVE, IV.tetrad and AIV.GT are close to the reference causal effect 2.01 and fall into the 95\% confidence interval.

\paragraph{Results on Schoolingreturns.} According to Fig.~\ref{fig:tworealworlddatasetswithoutIV}, we have the following findings: (1) the estimated result of TSLS is at the bottom of the 95\% confidence interval; (2) the estimated results of LSR and sisVIVE fall outside of the empirical interval. It is very likely that their assumptions have not been satisfied; (3) the estimated results of IV.tetrad and AIV.GT are in the empirical interval. They are consistency with the reference causal effect~\cite{verbeek2008guide}.  The consistency between the results IV.tetrad and AIV.GT is likely due to the reason that they have found proper conditional sets.

The experiments show that AIV-GT can obtain consistent estimations in both real-world datasets.    

\subsubsection{Comparing AIV.GT with the Estimators with Known IVs}
\label{subsubsec:withIV}
We add the four more comparison methods that require the given IVs, which are (1) TSLS.IV, TSLS with a given IV; (2) TSLS.CIV~\cite{glymour2012credible}, TSLS with a given CIV $S$ by conditioning on $\mathbf{X}\setminus\{S\}$; (3) FIVR, causal random forest for instrumental variable regression with a given CIV $S$ and conditioning on $\mathbf{X}\setminus\{S\}$~\cite{athey2019generalized}; (4) AIViP~\cite{cheng2022ancestral}, Ancestral IV estimator in PAG.

The two datasets have nominated IVs in the literature. The indicator of \textit{filaggrin} was used as an IV in VitD~\cite{martinussen2019instrumental} and Card~\cite{card1993using} used geographical proximity to a college, \ie \emph{nearcollege} as an IV in Schoolingreturns. All results of the above four estimators and AIV.GT are visualised in Fig.~\ref{fig:tworealworlddatasets}. 

\begin{figure}[t]
	\vskip -0.02in
	\begin{center}
		\centerline{\includegraphics[scale=0.435]{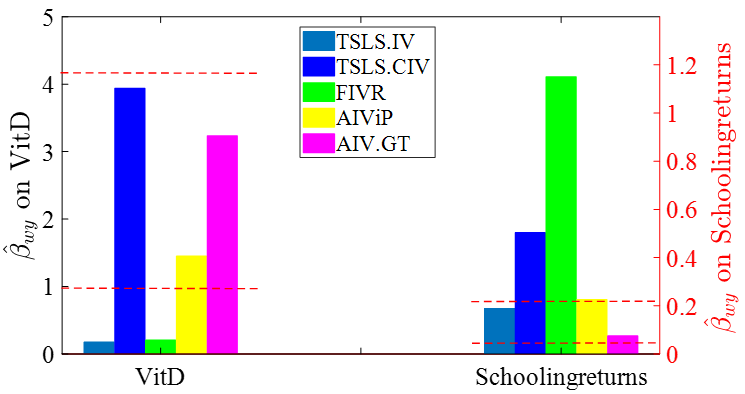}}
		\vskip -0.1in
		\caption{The experimental results of five estimators with a given IV for all comparison methods (AIV.GT does not use the given IV) on both real-world datasets. The two dotted lines represent empirically interval with 95\% confidence interval. }
		\label{fig:tworealworlddatasets}
	\end{center}
	\vskip -0.2in
\end{figure}

\paragraph{Results on VitD.} AIV.GT discovers $\{age, time\}$ as a pair of   AIVs. They are reasonable AIVs since they affect $W$ (vitd, vitamin D status) but do not directly affect $Y$ (death). From Fig.~\ref{fig:tworealworlddatasets}, we see that the results of TSLS.CIV, AIViP and AIV.GT are in the middle of 95\% empirical interval of the reference causal effect. 

\paragraph{Results on Schoolingreturns.} AIV.GT discovers $\{feducation, fameducation\}$ as a pair of   AIVs. They are valid IVs because father's educational attainment and family education class affect their child's education, but do not directly affect the child's income. From Fig.~\ref{fig:tworealworlddatasets}, we observe that the results of TSLS.IV, AIViP and AIV.GT are in the middle of the 95\% empirical interval of the reference causal effect.

In a word, AIV.GT, which does not need given AIVs, performs better or comparable with other methods which require a given IV. This shows the potential of the AIV.GT in a broader range of real-world applications.

\section{Related Work}
\label{sec:Relwork}
Latent variables are the major obstacle to estimating causal effect from observational data~\cite{greenland2003quantifying,pearl2009causality,chen2021causal}. When the treatment and outcome are confounded, IV methods~\cite{martens2006instrumental,angrist1995two,imbens2014instrumental,brito2002generalized,pearl2009causality,van2015efficiently,sokolovska2021role,xie2022testability} provide a solution. 

Some methods have been developed for standard IV based causal effect estimation when the IVs are given by domain experts~\cite{angrist1996identification,imbens2014instrumental}, such as the well-known two-stage least squares IV estimator~\cite{angrist1995two} (TSLS) which obtains causal effect using the ratio of two regression coefficients. Recently, Athey \etal~\cite{athey2019generalized} developed the generalised random forests to estimate conditional causal effects by using non-parametric quantile regression and instrumental variable regression (FIVR). The conditional causal effects can be aggregated into the average causal effect and FIVR has been compared in our experiments. We refer readers to~\cite{martens2006instrumental,hernan2006instruments,baiocchi2014instrumental} for a review of standard IV based methods. 

When a CIV is given, a proper conditioning set needs to be identified for unbiased causal effect estimation. Cheng \etal~\cite{cheng2022ancestral} have proposed a data-driven method AIViP for identifying such a conditioning set for casual effect estimation with a given CIV. Our work is different from these works since we focus on discovering AIVs and the corresponding conditioning set simultaneously from data and the proposed method is more general than the existing methods. 

CIV approach is very similar to the approach of covariate adjustment since both of them need to identify a proper conditioning set. However, there are essential differences between CIV and covariate adjustment. Most methods for covariate adjustment assume not a latent variable between $W$ on $Y$~\cite{pearl2009causality,maathuis2015generalized,perkovic2018complete}. There are four graphical criteria for identifying a proper conditioning set from a causal graph: back-door criterion~\cite{pearl2009causality}, adjustment criterion~\cite{shpitser2012validity}, generalised back-door criterion~\cite{maathuis2015generalized} and generalised adjustment criterion~\cite{perkovic2018complete}. There are some data-driven methods based on the four graphical criteria~\cite{maathuis2009estimating,entner2013data,cheng2022toward}. More detailed discussions for covariate selection can be found in~\cite{sauer2013review,witte2019covariate,guo2020survey}. There are no works for finding a conditioning set when the pair of $(W, Y)$ are confounded as discussed in this work.

\section{Conclusion}
\label{sec:con}
Estimating causal effects in the presence of latent variables is a challenging problem. IV is a well-known approach to address this challenge. However, most existing IV methods require strong domain knowledge or assumptions to determine an IV. This restricts the practical use of the IV approach. In this paper, we present the theory and a practical algorithm (AIV.GT) for finding valid AIVs and their corresponding conditioning sets from data, to enable data-driven causal effect estimation from data with latent variables. The experiments on synthetic datasets demonstrate that AIV.GT is able to address the challenges of latent variables and outperform the state-of-the-art causal effect estimators. The experimental results on two real-world datasets also show that AIV.GT achieves consistent results with empirical estimates in the literature, implying the practicability of AIV.GT in real-world applications.


\ifCLASSOPTIONcaptionsoff
\newpage
\fi

\bibliographystyle{IEEEtran}
\bibliography{IEEEabrv}

\begin{thebibliography}{10}
\providecommand{\url}[1]{#1}
\csname url@samestyle\endcsname
\providecommand{\newblock}{\relax}
\providecommand{\bibinfo}[2]{#2}
\providecommand{\BIBentrySTDinterwordspacing}{\spaceskip=0pt\relax}
\providecommand{\BIBentryALTinterwordstretchfactor}{4}
\providecommand{\BIBentryALTinterwordspacing}{\spaceskip=\fontdimen2\font plus
\BIBentryALTinterwordstretchfactor\fontdimen3\font minus
  \fontdimen4\font\relax}
\providecommand{\BIBforeignlanguage}[2]{{%
\expandafter\ifx\csname l@#1\endcsname\relax
\typeout{** WARNING: IEEEtran.bst: No hyphenation pattern has been}%
\typeout{** loaded for the language `#1'. Using the pattern for}%
\typeout{** the default language instead.}%
\else
\language=\csname l@#1\endcsname
\fi
#2}}
\providecommand{\BIBdecl}{\relax}
\BIBdecl

\bibitem{reiersol1945confluence}
O.~Reiers{\o}l, ``Confluence analysis by means of instrumental sets of
  variables,'' Ph.D. dissertation, Almqvist \& Wiksell, 1945.

\bibitem{imbens2015causal}
G.~W. Imbens and D.~B. Rubin, \emph{Causal {I}nference in {S}tatistics,
  {S}ocial, and {B}iomedical {S}ciences}.\hskip 1em plus 0.5em minus
  0.4em\relax Cambridge University Press, 2015.

\bibitem{blalock2017causal}
H.~M. Blalock, \emph{Causal {M}odels in the {S}ocial {S}ciences}.\hskip 1em
  plus 0.5em minus 0.4em\relax Routledge, 2017.

\bibitem{angrist1996identification}
J.~D. Angrist, G.~W. Imbens, and D.~B. Rubin, ``Identification of causal
  effects using instrumental variables,'' \emph{Journal of the American
  Statistical Association}, vol.~91, no. 434, pp. 444--455, 1996.

\bibitem{martens2006instrumental}
E.~P. Martens, W.~R. Pestman \emph{et~al.}, ``Instrumental variables:
  application and limitations,'' \emph{Epidemiology}, vol.~17, no.~3, pp.
  260--267, 2006.

\bibitem{pearl2009causality}
J.~Pearl, \emph{Causality}.\hskip 1em plus 0.5em minus 0.4em\relax Cambridge
  University Press, 2009.

\bibitem{kuroki2005instrumental}
M.~Kuroki and Z.~Cai, ``Instrumental variable tests for {D}irected {A}cyclic
  {G}raph {M}odels,'' in \emph{International Conference on Artificial
  Intelligence and Statistics}, pp. 190--197.

\bibitem{silva2017learning}
R.~Silva and S.~Shimizu, ``Learning instrumental variables with structural and
  non-gaussianity assumptions,'' \emph{Journal of Machine Learning Research},
  vol.~18, no. 120, pp. 1--49, 2017.

\bibitem{pearl1995testability}
J.~Pearl, ``On the testability of causal models with latent and instrumental
  variables,'' in \emph{Association for Uncertainty in Artificial
  Intelligence}, 1995, pp. 435--443.

\bibitem{chu2001semi}
T.~Chu, R.~Scheines, and P.~Spirtes, ``Semi-instrumental variables: a test for
  instrument admissibility,'' in \emph{Proceedings of the Seventeenth
  Conference on Uncertainty in Artificial Intelligence}, 2001, pp. 83--90.

\bibitem{kang2016instrumental}
H.~Kang, A.~Zhang \emph{et~al.}, ``Instrumental variables estimation with some
  invalid instruments and its application to {M}endelian randomization,''
  \emph{Journal of the American Statistical Association}, vol. 111, no. 513,
  pp. 132--144, 2016.

\bibitem{hartford2021valid}
J.~S. Hartford, V.~Veitch \emph{et~al.}, ``Valid causal inference with (some)
  invalid instruments,'' in \emph{International Conference on Machine
  Learning}.\hskip 1em plus 0.5em minus 0.4em\relax PMLR, 2021, pp. 4096--4106.

\bibitem{brito2002generalized}
C.~Brito and J.~Pearl, ``Generalized instrumental variables,'' in
  \emph{Proceedings of the Eighteenth Conference on Uncertainty in Artificial
  Intelligence}, 2002, pp. 85--93.

\bibitem{van2015efficiently}
B.~Van~der Zander, M.~Li{\'s}kiewicz, and J.~Textor, ``Efficiently finding
  conditional instruments for causal inference,'' in \emph{International Joint
  Conference on Artificial Intelligence}.\hskip 1em plus 0.5em minus
  0.4em\relax AAAI Press, 2015, pp. 3243--3249.

\bibitem{richardson2002ancestral}
T.~Richardson and P.~Spirtes, ``Ancestral graph markov models,'' \emph{The
  Annals of Statistics}, vol.~30, no.~4, pp. 962--1030, 2002.

\bibitem{zhang2008causal}
J.~Zhang, ``Causal reasoning with ancestral graphs,'' \emph{Journal of Machine
  Learning Research}, vol.~9, no.~7, pp. 1437--1474, 2008.

\bibitem{spirtes2000causation}
P.~Spirtes, C.~N. Glymour \emph{et~al.}, \emph{Causation, {P}rediction, and
  {S}earch}.\hskip 1em plus 0.5em minus 0.4em\relax MIT Press, 2000.

\bibitem{zhang2008completeness}
J.~Zhang, ``On the completeness of orientation rules for causal discovery in
  the presence of latent confounders and selection bias,'' \emph{Artificial
  Intelligence}, vol. 172, no. 16-17, pp. 1873--1896, 2008.

\bibitem{perkovic2018complete}
E.~Perkovi{\'c}, J.~Textor, and M.~Kalisch, ``Complete graphical
  characterization and construction of adjustment sets in markov equivalence
  classes of ancestral graphs,'' \emph{Journal of Machine Learning Research},
  vol.~18, pp. 1--62, 2018.

\bibitem{abadie2003semiparametric}
A.~Abadie, ``Semiparametric instrumental variable estimation of treatment
  response models,'' \emph{Journal of Econometrics}, vol. 113, no.~2, pp.
  231--263, 2003.

\bibitem{athey2019generalized}
S.~Athey, J.~Tibshirani, and S.~Wager, ``Generalized random forests,''
  \emph{The Annals of Statistics}, vol.~47, no.~2, pp. 1148--1178, 2019.

\bibitem{bowden1990instrumental}
R.~J. Bowden and D.~A. Turkington, \emph{Instrumental {V}ariables}.\hskip 1em
  plus 0.5em minus 0.4em\relax Cambridge University Press, 1990, no.~8.

\bibitem{hernan2006instruments}
M.~A. Hern{\'a}n and J.~M. Robins, ``Instruments for causal inference: an
  epidemiologist's dream?'' \emph{Epidemiology}, pp. 360--372, 2006.

\bibitem{cheng2022ancestral}
D.~Cheng, J.~Li \emph{et~al.}, ``Ancestral instrument method for causal
  inference without complete knowledge,'' in \emph{International Joint
  Conference on Artificial Intelligence}, 2022.

\bibitem{aliferis2010local}
C.~F. Aliferis, A.~Statnikov, I.~Tsamardinos, S.~Mani, and X.~D. Koutsoukos,
  ``Local causal and markov blanket induction for causal discovery and feature
  selection for classification part i: Algorithms and empirical evaluation,''
  \emph{Journal of Machine Learning Research}, vol.~11, no. Jan, pp. 171--234,
  2010.

\bibitem{kalisch2012causal}
M.~Kalisch, M.~M{\"a}chler \emph{et~al.}, ``Causal inference using graphical
  models with the {R} package pcalg,'' \emph{Journal of Statistical Software},
  vol.~47, no.~11, pp. 1--26, 2012.

\bibitem{scutari2009learning}
M.~Scutari, ``Learning {B}ayesian {N}etworks with the bnlearn {R} package,''
  \emph{Journal of Statistical Software}, vol.~35, no. i03, pp. 1--22, 2010.

\bibitem{colombo2012learning}
D.~Colombo, M.~H. Maathuis \emph{et~al.}, ``Learning high-dimensional directed
  acyclic graphs with latent and selection variables,'' \emph{The Annals of
  Statistics}, vol.~40, no.~1, pp. 294--321, 2012.

\bibitem{wishart1928sampling}
J.~Wishart, ``Sampling errors in the theory of two factors,'' \emph{British
  Journal of Psychology}, vol.~19, no.~2, p. 180, 1928.

\bibitem{card1993using}
D.~Card, ``Using geographic variation in college proximity to estimate the
  return to schooling,'' 1993.

\bibitem{sjolander2019instrumental}
A.~Sjolander and T.~Martinussen, ``Instrumental variable estimation with the
  {R} package ivtools,'' \emph{Epidemiologic Methods}, vol.~8, no.~1, 2019.

\bibitem{angrist1995two}
J.~D. Angrist and G.~W. Imbens, ``Two-stage least squares estimation of average
  causal effects in models with variable treatment intensity,'' \emph{Journal
  of the American statistical Association}, vol.~90, no. 430, pp. 431--442,
  1995.

\bibitem{greene2003econometric}
W.~H. Greene, \emph{Econometric {A}nalysis}.\hskip 1em plus 0.5em minus
  0.4em\relax Pearson Education India, 2003.

\bibitem{martinussen2019instrumental}
T.~Martinussen \emph{et~al.}, ``Instrumental variables estimation under a
  structural {C}ox model,'' \emph{Biostatistics}, vol.~20, no.~1, pp. 65--79,
  2019.

\bibitem{verbeek2008guide}
M.~Verbeek, \emph{A {G}uide to {M}odern {E}conometrics}.\hskip 1em plus 0.5em
  minus 0.4em\relax John Wiley \& Sons, 2008.

\bibitem{glymour2012credible}
M.~M. Glymour, E.~J. Tchetgen~Tchetgen, and J.~M. Robins, ``Credible mendelian
  randomization studies: approaches for evaluating the instrumental variable
  assumptions,'' \emph{American journal of epidemiology}, vol. 175, no.~4, pp.
  332--339, 2012.

\bibitem{greenland2003quantifying}
S.~Greenland, ``Quantifying biases in causal models: classical confounding vs
  collider-stratification bias,'' \emph{Epidemiology}, pp. 300--306, 2003.

\bibitem{chen2021causal}
W.~Chen, R.~Cai, K.~Zhang, and Z.~Hao, ``Causal discovery in linear
  non-gaussian acyclic model with multiple latent confounders,'' \emph{IEEE
  Transactions on Neural Networks and Learning Systems}, pp. 1--12, 2021.

\bibitem{imbens2014instrumental}
G.~W. Imbens, ``Instrumental {V}ariables: {A}n {E}conometrician's
  {P}erspective,'' \emph{Statistical Science}, vol.~29, no.~3, pp. 323--358,
  2014.

\bibitem{sokolovska2021role}
N.~Sokolovska and P.-H. Wuillemin, ``The role of instrumental variables in
  causal inference based on independence of cause and mechanism,''
  \emph{Entropy}, vol.~23, no.~8, p. 928, 2021.

\bibitem{xie2022testability}
F.~Xie, Y.~He \emph{et~al.}, ``Testability of instrumental variables in linear
  non-gaussian acyclic causal models,'' \emph{Entropy}, vol.~24, no.~4, p. 512,
  2022.

\bibitem{baiocchi2014instrumental}
M.~Baiocchi, J.~Cheng, and D.~S. Small, ``Instrumental variable methods for
  causal inference,'' \emph{Statistics in Medicine}, vol.~33, no.~13, pp.
  2297--2340, 2014.

\bibitem{maathuis2015generalized}
M.~H. Maathuis, D.~Colombo \emph{et~al.}, ``A generalized back-door
  criterion,'' \emph{The Annals of Statistics}, vol.~43, no.~3, pp. 1060--1088,
  2015.

\bibitem{shpitser2012validity}
I.~Shpitser, T.~J. VanderWeele, and J.~M. Robins, ``On the validity of
  covariate adjustment for estimating causal effects,'' in \emph{International
  Conference on Uncertaintyin Artificial Intelligence}.\hskip 1em plus 0.5em
  minus 0.4em\relax {AUAI} Press, 2010, pp. 527--536.

\bibitem{maathuis2009estimating}
M.~H. Maathuis, M.~Kalisch, and P.~B{\"u}hlmann, ``Estimating high-dimensional
  intervention effects from observational data,'' \emph{The Annals of
  Statistics}, vol.~37, no.~6A, pp. 3133--3164, 2009.

\bibitem{entner2013data}
D.~Entner, P.~Hoyer, and P.~Spirtes, ``Data-driven covariate selection for
  nonparametric estimation of causal effects,'' in \emph{Proceedings of
  Artificial Intelligence and Statistics}, 2013, pp. 256--264.

\bibitem{cheng2022toward}
D.~Cheng, J.~Li \emph{et~al.}, ``Toward unique and unbiased causal effect
  estimation from data with hidden variables,'' \emph{IEEE Transactions on
  Neural Networks and Learning Systems}, 2022.

\bibitem{sauer2013review}
B.~C. Sauer, M.~A. Brookhart, J.~Roy, and T.~VanderWeele, ``A review of
  covariate selection for non-experimental comparative effectiveness
  research,'' \emph{Pharmacoepidemiology and Drug Safety}, vol.~22, no.~11, pp.
  1139--1145, 2013.

\bibitem{witte2019covariate}
J.~Witte and V.~Didelez, ``Covariate selection strategies for causal inference:
  Classification and comparison,'' \emph{Biometrical Journal}, vol.~61, no.~5,
  pp. 1270--1289, 2019.

\bibitem{guo2020survey}
R.~Guo, L.~Cheng, J.~Li, P.~R. Hahn, and H.~Liu, ``A survey of learning
  causality with data: Problems and methods,'' \emph{ACM Computing Surveys
  (CSUR)}, vol.~53, no.~4, pp. 1--37, 2020.

\end{thebibliography}

\end{document}